\documentclass{article}

\usepackage[preprint]{neurips_2026}

\usepackage[utf8]{inputenc}
\usepackage[T1]{fontenc}
\usepackage{hyperref}
\usepackage{url}
\usepackage{microtype}
\usepackage{graphicx}
\usepackage{subcaption}
\usepackage{wrapfig}
\usepackage{placeins}
\usepackage{booktabs}
\usepackage{amsfonts}
\usepackage{nicefrac}
\usepackage{amsmath}
\usepackage{amssymb}
\usepackage{mathtools}
\usepackage{amsthm}
\usepackage[table]{xcolor}
\usepackage{bbm}
\usepackage{multirow}
\usepackage{caption}

\usepackage[capitalize,noabbrev]{cleveref}
\setcitestyle{round}

\theoremstyle{plain}
\newtheorem{theorem}{Theorem}[section]
\newtheorem{proposition}[theorem]{Proposition}

\theoremstyle{definition}

\theoremstyle{remark}

\usepackage[disable,textsize=tiny]{todonotes}

\title{Beyond the Target: From Imitation to Collaboration in Speculative Decoding}

\author{%
  Jinze Li$^{1,2}$, Yixing Xu$^1$, Guanchen Li$^1$, Jinfeng Xu$^2$, Shuo Yang$^2$, Yang Zhang$^3$\\
   \textbf{Xuanwu Yin$^1$, Dong Li$^1$, Edith C.H. Ngai$^{2}$\thanks{Corresponding author.}, Emad Barsoum$^1$} \\
  $^1$Advanced Micro Devices, Inc., Beijing, China \\
  $^2$The University of Hong Kong \,  $^3$University of North Texas \\
  \texttt{\{lijinze-hku, shuo.yang, jinfeng\}@connect.hku.hk} \\
  \texttt{\{yixing.xu, Xuanwu.Yin, d.li, emad.barsoum\}@amd.com} \\
  \texttt{yang.zhang@unt.edu}, \,
  \texttt{chngai@eee.hku.hk} \\
}
\hypersetup{hidelinks}
\begin{document}

\maketitle

\begin{abstract}
Speculative decoding (SPD) accelerates large language model (LLM) inference by letting a smaller draft model propose multiple future tokens that are verified in parallel by a larger target model.
The dominant SPD paradigm treats the target model as the sole reliable teacher, accepting a draft token only when it exactly matches the target prediction.
This design implicitly assumes that the target is always the better choice at every position. In practice, this assumption does not hold. 
Although the draft is the weaker model overall, it is not uniformly inferior at the token level. In a meaningful fraction of cases where draft and target disagree, the draft's choice is the one that leads to the correct final answer.
Inspired by this, we introduce \textbf{Collaborative Speculative Decoding (CoSpec)}, a generalization of SPD that no longer treats the target model as the sole token-level authority.
CoSpec trains an arbitration policy via reinforcement learning to decide whether to accept tokens from the draft or target model, selectively accepting draft tokens at mismatches when doing so is likely to yield a correct final answer.
Experimental results show that CoSpec maintains substantial speedups while surpassing target-only performance. By shifting the emphasis from imitation to collaboration, CoSpec suggests a new perspective on speculative decoding.
\end{abstract}

\section{Introduction}

\begin{wrapfigure}[14]{r}{0.48\textwidth}
    \vspace{-8pt}
    \centering
    \includegraphics[width=0.48\textwidth]{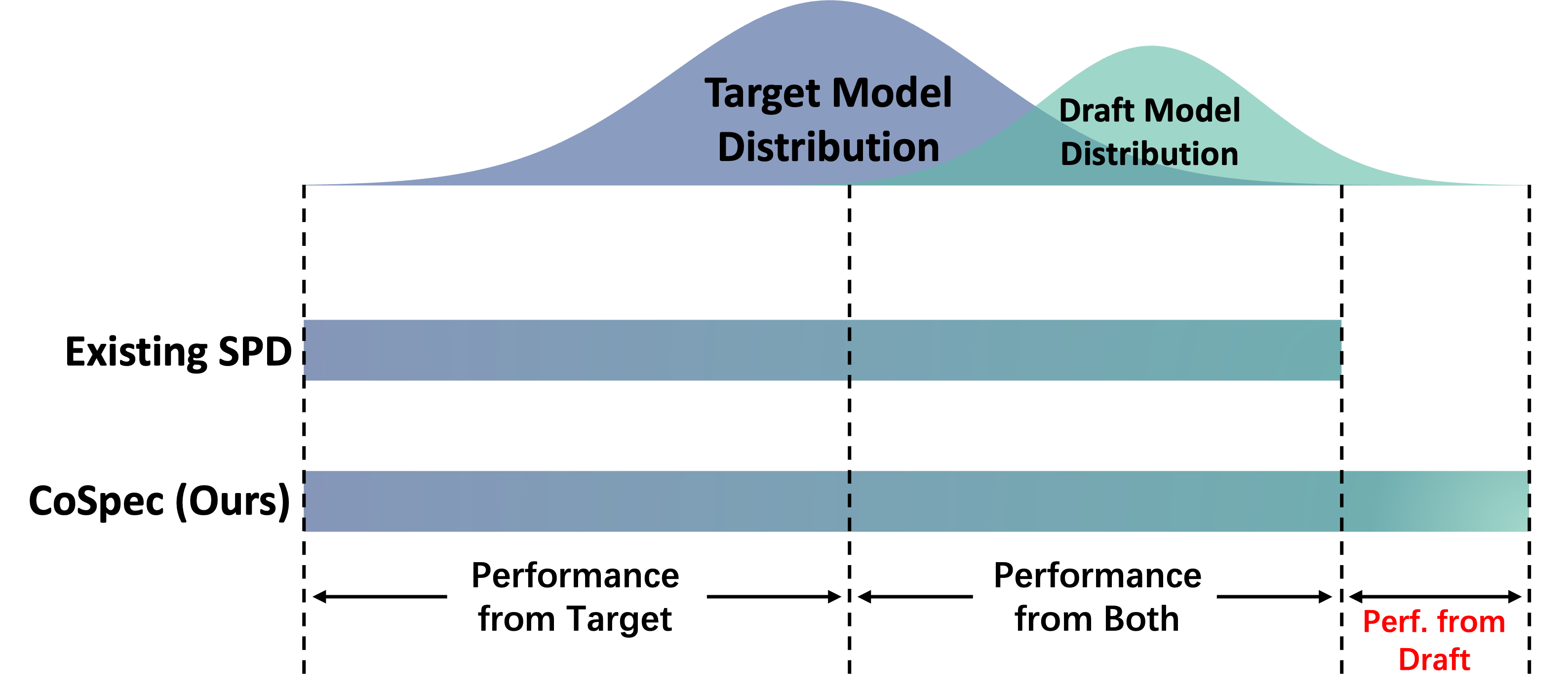}
    \caption{Illustration of the paradigm shift from imitation to collaboration. While existing SPD methods are strictly confined to the target model's distribution, CoSpec leverages the complementary strengths of the draft model (highlighted in \textcolor{red}{red}) to surpass the performance ceiling of the target model alone.}
    \label{fig:CoSpec_comparison}
    \vspace{-10pt}
\end{wrapfigure}

Large language models deliver strong performance across reasoning, code, dialogue, and knowledge-intensive tasks~\citep{llama3,deepseek-r1,qwen3}.
However, their deployment cost is dominated by autoregressive decoding, where each generated token requires a separate forward pass through the large model.
Speculative decoding addresses this bottleneck by separating token proposal from token verification: a smaller draft model first proposes a block of future tokens, and the larger target model then verifies the block in a single parallel pass~\citep{leviathan2023fast,chen2023accelerating}.
Because this draft-and-verify procedure can reduce latency while preserving the target model distribution in the standard sampling setting, it has become a central strategy for inference acceleration.

The dominant design principle in speculative decoding is therefore accuracy preservation with respect to the target model.
Under this principle, strict methods accept a draft token only when it exactly matches the target prediction.
Subsequent learned drafting methods, such as Medusa, EAGLE, and Gumiho, improve speed by making proposed tokens more likely to pass target verification~\citep{cai2024medusa,li2024eagle,li2024eagle2,li2025eagle,li2025gumihohybridarchitectureprioritize}.
More recent loosely verified methods further relax exact identity by accepting tokens that are semantically compatible with the target or unlikely to harm target quality~\citep{bachmann2025judge,holsman2025fuzzy,li2025fly,yoon2025selfjudge,wang2025reflective,pivot2025}.
Despite these differences, these approaches share the same underlying assumption: the target model remains the reference point for deciding whether a draft token is acceptable.
In other words, the draft is useful only insofar as it imitates, preserves, or safely approximates the target.

This paper starts from a different observation.
Although a draft model is weaker than the target model in aggregate, it is not uniformly inferior on every example or every local decision.
In reasoning and code generation, for example, two models from the same family can make different intermediate choices that lead to different final answers.
As a result, a draft--target mismatch is not always a draft error.
It can be a harmful deviation, a harmless paraphrase, or a useful alternative that avoids a target mistake.
Therefore, if the verifier always falls back to the target at such states, it may discard useful information carried by the draft model and miss the opportunity to exploit non-nested model strengths.

Motivated by this observation, we propose \textbf{Collaborative Speculative Decoding} (CoSpec), a speculative decoding framework that treats draft--target disagreement as an opportunity for collaboration rather than merely as a verification failure.
CoSpec retains the original goal of acceleration, but changes the role of the draft model.
Instead of requiring the draft to only imitate or approximate the target, CoSpec allows the draft to contribute when its continuation is more useful for the final task outcome.
In this sense, CoSpec is not designed as a distribution-preserving sampler at selected mismatches.
Its objective is final task utility under an acceleration constraint, which is especially relevant for math reasoning, code generation, factual QA, and benchmark-driven deployment, where outputs are judged primarily by answer correctness or pass rate.
From this perspective, a mismatch state naturally induces two possible continuations: one obtained by following the draft token and the other obtained by following the target token.
A target-faithful verifier is restricted to the target branch, whereas a collaborator can choose the branch with higher expected final utility.
Whenever the draft branch is better on a nonzero subset of target failure states, collaboration can create a positive complementarity gap over target-only decoding.
\Cref{fig:CoSpec_comparison} illustrates this shift from target-bounded imitation to collaborative decoding that can exploit useful draft alternatives beyond the target-only ceiling.

Concretely, CoSpec keeps the computational structure of speculative decoding: the draft model proposes a block of future tokens, and the target model verifies this block in a single parallel pass.
Matched tokens are still force-accepted, so the standard efficiency pattern of speculative decoding is preserved.
The key change happens at disagreement points.
Instead of rejecting the draft at the first mismatch or applying a fixed relaxed verification rule, CoSpec invokes a learned arbitration policy.
The policy observes the prompt, the verified prefix, the draft block, and the target verification tokens, and then decides whether to keep the draft token or fall back to the target token at each mismatch.
By selectively permitting informative mismatches, the arbitration policy increases the number of tokens accepted per verification round, better leveraging the draft model’s capacity.
Empirically, CoSpec improves target-level accuracy while achieving substantial speedups, consistently attaining state-of-the-art performance.

Our contributions are summarized as follows:
\begin{itemize}
\item We formulate a generalized paradigm for speculative decoding, shifting the relationship between draft and target models from pure imitation to active collaboration.

\item Building on this paradigm, we introduce CoSpec, which leverages an arbitration policy to determine which tokens to adopt when mismatches occur between the draft and target models, achieving high final accuracy while preserving acceleration.

\item Through extensive experimental evaluations, we demonstrate the effectiveness of our method. CoSpec consistently achieves state-of-the-art performance across benchmarks.
\end{itemize}

\section{Related Work}

\paragraph{Target preserving speculative decoding.}
Speculative decoding was introduced as a lossless acceleration method where a draft model proposes tokens and a target model verifies them in parallel~\citep{leviathan2023fast,chen2023accelerating}.
Subsequent work improves the draft proposal process using learned heads, recurrent or feature based drafters, retrieval, early exit, dynamic trees, block verification, or systems optimizations~\citep{cai2024medusa,li2024eagle,li2024eagle2,li2025eagle,he2023rest,he2023retrieval,elhoushi2024self,chen2024csdraft,sun2025block,li2025gumihohybridarchitectureprioritize}.
These methods improve acceptance length or wall clock throughput while preserving the target distribution or the target output quality.
CoSpec is complementary to this line of work because it keeps the draft-and-verify execution pattern but changes the decision problem at mismatches.

\paragraph{Loosely verified and utility biased speculative decoding.}
A growing line of work relaxes exact token matching.
Judge Decoding trains a verifier to identify mismatched drafts that remain acceptable~\citep{bachmann2025judge}.
Fuzzy speculative decoding provides a tunable accuracy and runtime tradeoff~\citep{holsman2025fuzzy}.
FLy uses target corrective behavior to accept semantically valid mismatches without training~\citep{li2025fly}.
SelfJudge, reflective verification, sequential approximate verification, and pivot aware decoding further reduce unnecessary rejections through learned, self supervised, semantic, or utility aware tests~\citep{yoon2025selfjudge,wang2025reflective,zhong2025sprinter,pivot2025}.
These methods mainly ask whether a draft mismatch can be accepted while preserving or approximating target utility.
Several recent methods make the utility objective more explicit.
Reward Guided Speculative Decoding uses a process reward model to bias reasoning trajectories toward high reward steps~\citep{liao2025rsd}, and alignment augmented speculative decoding (AASD) uses retrieval based drafts with alignment sampling and conditional verification to improve long context generation quality~\citep{wang2025aasd}.
This line of work aims to achieve higher speedup while minimally affecting target-model performance, whereas our method aims to obtain strong acceleration while improving the target model’s task performance.

\paragraph{Cascades and multi model inference.}
Language model cascades and speculative cascades also combine models to improve cost and quality tradeoffs~\citep{chen2023frugalgpt,ong2024routellm,ding2024hybridllm,jiang2023llmblender,wang2024moa,chen2024csdraft,narasimhan2025fastercascades}.
They typically use input level routing, response level ranking, or multi model aggregation to decide which model output to use.
CoSpec instead operates at token mismatch states inside speculative decoding rounds.
The target model is already invoked in parallel verification mode, and the learned decision is whether to use the draft or target token at a local disagreement.
This gives CoSpec a different operating point: it can exploit subproblem level complementarity while retaining speculative batching.

\section{Method}
\label{sec:method}

\subsection{Preliminaries}
\label{sec:prelim}

\paragraph{Notation.}
Let $x$ denote the input prompt and $\mathcal{V}$ the vocabulary.
We consider a large target model $\mathcal{M}_{\mathrm{T}}$ and a smaller draft model $\mathcal{M}_{\mathrm{D}}$ that share $\mathcal{V}$.
Both models decode autoregressively.

\paragraph{Speculative decoding.}
Speculative decoding operates in rounds.
Let $y^{(<r)}$ be the verified prefix at the beginning of round $r$ and let $K$ be the draft length.
The draft model proposes
\begin{equation}
\hat{\mathbf{y}}^{(r)}=(\hat{y}^{(r)}_1,\dots,\hat{y}^{(r)}_K)
\end{equation}
conditioned on $(x,y^{(<r)})$.
The target model then performs one teacher-forced forward pass over $[x, y^{(<r)}, \hat{y}^{(r)}_1, \dots, \hat{y}^{(r)}_K]$, producing one next-token prediction at each of the $K+1$ rightmost positions:
\begin{equation}
(\tilde{y}^{(r)}_1,\dots,\tilde{y}^{(r)}_{K+1}).
\end{equation}
The first $K$ predictions verify the draft positions and $\tilde{y}^{(r)}_{K+1}$ is the bonus token after the final draft token.
We define
\begin{equation}
\Delta^{(r)}_i=\mathbbm{1}[\hat{y}^{(r)}_i=\tilde{y}^{(r)}_i].
\end{equation}
Exact-match speculative decoding accepts draft tokens until the first mismatch and then emits the target token at that position.
The number of generated tokens in round $r$ is
\begin{equation}
\label{eq:strict-spd-len}
s^{(r)}=
\begin{cases}
\min\{i\mid \Delta^{(r)}_i=0\}, & \exists i: \Delta^{(r)}_i=0, \\
K+1, & \text{otherwise}.
\end{cases}
\end{equation}
For a run of $R$ rounds, the mean accepted length is
\begin{equation}
\label{eq:mean-tau}
\tau=\frac{1}{R}\sum_{r=1}^{R}s^{(r)}.
\end{equation}
We report speedup as the end-to-end throughput ratio relative to autoregressive target decoding under the same hardware, batching, prompt, and output length setting.

\paragraph{Stochastic decoding.}
Let $T$ denote the sampling temperature.
For $T=1$ experiments, distribution-preserving baselines use standard speculative sampling with acceptance ratio correction.
CoSpec does not claim target distribution preservation in the stochastic setting.
It uses the sampled draft candidate and the corresponding target verification candidate as two actions for task utility arbitration.
This design is appropriate when the evaluation metric is final task utility rather than exact sampling from the target distribution.

\begin{figure}[t]
  \centering
  \includegraphics[width=0.95\textwidth]{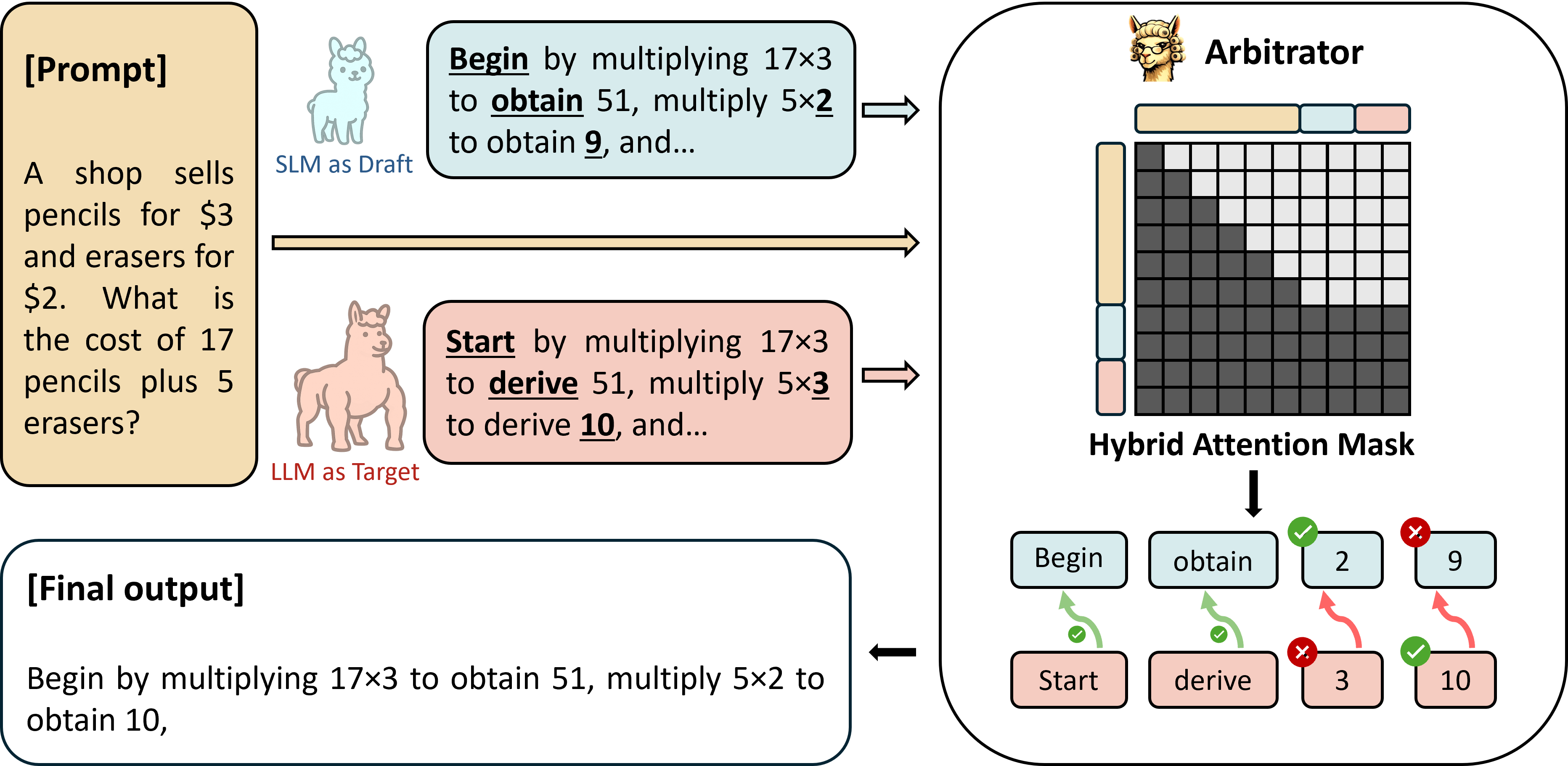}
  \caption{Overview of CoSpec. The draft proposes a block and the target verifies it in one teacher-forced pass. A learned arbitrator receives the context, draft block, and target verification tokens. It uses a hybrid attention mask and decides at each mismatch whether to keep the draft token or fall back to the target token.}
  \label{fig:overview}
  \vspace{-1.5em}
\end{figure}

\subsection{A Utility View of Collaborative SPD}
\label{sec:utility_view}

Let \(U(x,y)\in[0,1]\) be the task utility of a completed response, such as exact answer correctness or pass at one. At a reached mismatch state
\(m=(x,y^{(<r)},i,\hat y_{1:K}^{(r)},\tilde y_{1:K+1}^{(r)},\Delta_{1:K}^{(r)})\), define two counterfactual branch utilities:
\begin{equation}
u_D(m)=\mathbb{E}[U(x,Y)\mid m,\text{force }\hat y_i^{(r)}],\quad
u_T(m)=\mathbb{E}[U(x,Y)\mid m,\text{force }\tilde y_i^{(r)}].
\end{equation}
Here \(Y\) is the completed response. The draft branch appends \(\hat y_i^{(r)}\) and continues scanning the current block, whereas the target branch appends \(\tilde y_i^{(r)}\), terminates the current round, and resumes CoSpec from the next round. The expectation is over the fixed reference continuation policy \(\pi_{\rm ref}\) under these dynamics.
A target-faithful verifier uses the target branch when a mismatch is rejected.
An oracle arbitrator chooses the branch with larger expected utility.

\begin{proposition}[Local complementarity gap]
\label{prop:complementarity_gap}
Fix \(\pi_{\rm ref}\), and let \(\mathcal S\) be the reached-mismatch-state distribution induced by the prompts, draft/target decoding, and \(\pi_{\rm ref}\). The expected local oracle utility exceeds the target-faithful branch utility by
\begin{equation}
\mathbb{E}_{m\sim\mathcal S}\!\left[\max\{u_D(m),u_T(m)\}\right]
-\mathbb{E}_{m\sim\mathcal S}\!\left[u_T(m)\right]
=
\mathbb{E}_{m\sim\mathcal S}\!\left[(u_D(m)-u_T(m))_+\right],
\end{equation}
where \((z)_+ := \max(z,0)\). The gap is positive exactly when the draft branch has larger expected utility than the target branch on a nonzero measure of reached mismatch states.
\end{proposition}

\begin{proof}
The identity follows from \(\max\{u_D,u_T\}=u_T+(u_D-u_T)_+\) at each reached state \(m\). Taking expectation over \(m\sim\mathcal S\) gives the result.
\end{proof}

Proposition~\ref{prop:complementarity_gap} gives the central target for CoSpec.
The aim is not to make every draft token pass target verification.
The aim is to approximate the utility-maximizing branch selector while preserving the parallel verification structure that gives speculative decoding its speed advantage.
\Cref{fig:overview} gives an overview of this procedure, where the arbitrator compares the draft and target candidates within each speculative round and decides which branch to follow at mismatches.

\subsection{Collaborative Arbitration}
\label{sec:cospec_algo}

CoSpec learns an arbitration policy $\pi_\theta$ that decides whether a mismatch should keep the draft token or fall back to the target token.
Let
\begin{equation}
c^{(r)}=[x,y^{(<r)}]
\end{equation}
be the prompt and verified prefix.
For round $r$, the arbitrator input is
\begin{equation}
\label{eq:arb_input}
q^{(r)}=\left[c^{(r)},\mathrm{[SEP]},\hat{y}^{(r)}_{1:K},\mathrm{[SEP]},\tilde{y}^{(r)}_{1:K},\mathrm{[SEP]}\right].
\end{equation}
The bonus token $\tilde{y}^{(r)}_{K+1}$ is omitted because arbitration only acts on the $K$ paired draft--target candidates.
Let $l_c \triangleq |c^{(r)}|$ and let $l = l_c + 2K + 3$ be the total input length.
We apply a hybrid attention pattern: the context portion attends causally within itself, while the draft--target region uses fully bidirectional attention within that region and may also attend to the entire context.
Information flows from context to draft/target and within draft/target in both directions, but never from draft/target back to context.
The draft and target models themselves remain causal and frozen. The hybrid mask only changes the diagnostic view available to the arbitrator.

The policy backbone is a language model with LoRA parameters $\psi$.
Let $\mathcal{A}$ denote the hybrid attention mask described above.
It produces hidden states
\begin{equation}
h^{(r)}=\mathrm{LM}_{\psi}(q^{(r)},\mathcal{A}).
\end{equation}
For each draft position $i \in \{1, \dots, K\}$, a lightweight head $g_\phi$ maps the hidden state at the position of $\hat{y}^{(r)}_i$ within $q^{(r)}$ to a scalar logit:
\begin{equation}
z^{(r)}_i=g_\phi\!\left(h^{(r)}_{l_c+1+i}\right).
\end{equation}
The action $a^{(r)}_i\in\{0,1\}$ denotes whether to accept the draft token:
\begin{equation}
a^{(r)}_i=1 \Rightarrow \hat{y}^{(r)}_i,
\qquad
a^{(r)}_i=0 \Rightarrow \tilde{y}^{(r)}_i.
\end{equation}
With $\theta = \{\psi, \phi\}$ collecting all trainable parameters, the acceptance probability is
\begin{equation}
\label{eq:policy-prob}
\pi_\theta(a^{(r)}_i=1\mid q^{(r)})=\sigma(\tilde{z}^{(r)}_i),
\end{equation}
where $\sigma(\cdot)$ denotes the logistic sigmoid and
\begin{equation}
\tilde{z}^{(r)}_i=
\begin{cases}
+\infty, & \Delta^{(r)}_i=1, \\
z^{(r)}_i, & \Delta^{(r)}_i=0.
\end{cases}
\end{equation}
Matched tokens are forced-accepted and do not create trainable decisions.
At mismatches, CoSpec thresholds or samples from the learned policy.

Within a round, CoSpec stops at the first rejected mismatch:
\begin{equation}
\label{eq:cospec_length}
s^{(r)}_{\mathrm{CoSpec}}=
\begin{cases}
\min\{i\mid \Delta^{(r)}_i=0 \wedge a^{(r)}_i=0\}, & \exists i: \Delta^{(r)}_i=0 \wedge a^{(r)}_i=0, \\
K+1, & \text{otherwise}.
\end{cases}
\end{equation}
Thus $s^{(r)}_{\mathrm{CoSpec}} \in \{1, \dots, K+1\}$ is the round-terminating index, equal to the number of tokens emitted in round $r$.
If the first rejected mismatch occurs at position $i$, the output appends all tokens before $i$ and then appends $\tilde{y}^{(r)}_i$.
If no mismatch is rejected, the output appends all $K$ draft tokens plus the bonus target token.
This keeps the standard round structure while allowing draft tokens to pass through selected mismatches.

\subsection{Reinforcement Learning Training}
\label{sec:rl_train}

We optimize the arbitrator with a two-stage pipeline.
The supervised stage initializes $\pi_\theta$ from judge-labeled mismatch preferences over rollouts of vanilla SPD. We defer its details to Appendix~\ref{app:implementation} and denote the resulting policy by $\pi_{\mathrm{ref}}$.
The reinforcement learning stage, described below, optimizes decisions under full CoSpec rollouts.
All target and draft model weights remain frozen, only the arbitrator LoRA parameters $\psi$ and head $\phi$ are trained.

\paragraph{Rollouts and actions.}
For each prompt $x$, we generate $M$ CoSpec rollouts
\begin{equation}
\mathcal{G}(x)=\{\xi^{(1)},\xi^{(2)},\dots,\xi^{(M)}\},
\end{equation}
where each rollout $\xi^{(m)} = \{(\hat{\mathbf{y}}^{(r)}, \tilde{\mathbf{y}}^{(r)}, \mathbf{a}^{(r)})\}_{r=1}^{R^{(m)}}$ records the draft block, target verification, and arbitrator actions across all rounds until EOS.
At training time, mismatch actions are sampled from $\pi_\theta$.
At evaluation time, we use a deterministic threshold $\lambda$:
\begin{equation}
\label{eq:infer_thre}
a^{(r)}_i=\mathbbm{1}[\pi_\theta(a^{(r)}_i=1\mid q^{(r)})>\lambda].
\end{equation}
Only mismatch decisions up to the first rejection can affect the current round.
For rollout $\xi^{(m)}$, the effective action set in round $r$ is
\begin{equation}
\mathcal{I}^{(m)}_r=
\{i: \Delta^{(r)}_i=0 \wedge i\le \min\{s^{(r)}_{\mathrm{CoSpec}},K\}\}.
\end{equation}

\paragraph{Reward shaping.}
The terminal correctness signal is $\mathrm{Corr}(\xi^{(m)})\in\{0,1\}$.
For correct rollouts, we reward accepting useful span beyond the first mismatch.
Let
\begin{equation}
i^\star_r=\min\{i\in\{1,\dots,K\}\mid \Delta^{(r)}_i=0\}
\end{equation}
when a mismatch exists, and define
\begin{equation}
\rho_r=\frac{s^{(r)}_{\mathrm{CoSpec}}-i^\star_r}{(K+1)-i^\star_r}.
\end{equation}
For correct rollouts, each effective mismatch decision in round $r$ receives $g_{r,i}=\rho_r$.
If round $r$ has no mismatch, $\mathcal{I}^{(m)}_r = \emptyset$ and the round contributes no terms to the rollout return.

For incorrect rollouts, CoSpec assigns negative credit to two decision types.
The first type is accepting a draft token that is strongly disfavored by the target.
The second type is rejecting a mismatch and thereby committing to a target branch that can also be wrong.
Let $i^{(m)}_{\mathrm{rej},r}$ be the first rejected mismatch in round $r$, using $K+1$ as the sentinel for no rejection.
Define accepted and rejected decision sets

\begin{equation}
\mathcal{S}^{(m)}_{r,\mathrm{acc}}=\{i\in\mathcal{I}^{(m)}_r\mid a^{(r)}_i=1, i<i^{(m)}_{\mathrm{rej},r}\},
\end{equation}
\begin{equation}
\mathcal{S}^{(m)}_{r,\mathrm{rej}}=
\begin{cases}
\{i^{(m)}_{\mathrm{rej},r}\}, & i^{(m)}_{\mathrm{rej},r}\le K, \\
\emptyset, & i^{(m)}_{\mathrm{rej},r}=K+1.
\end{cases}
\end{equation}

Let
\begin{equation}
\ell^{(r)}_i=\log p_{\mathcal{M}_{\mathrm{T}}}(\hat{y}^{(r)}_i\mid x,y^{(<r)},\hat{y}^{(r)}_{<i}).
\end{equation}
The failure reward is
\begin{equation}
\label{eq:failure_penalty}
g_{r,i}=-\alpha |\ell^{(r)}_i|\mathbbm{1}[i\in\mathcal{S}^{(m)}_{r,\mathrm{acc}}]-\beta \mathbbm{1}[i\in\mathcal{S}^{(m)}_{r,\mathrm{rej}}],
\end{equation}
where $|\ell^{(r)}_i| = -\ell^{(r)}_i$ is the surprisal of $\hat{y}^{(r)}_i$ under the target distribution.
The target log-probability term is a risk regularizer, not a target-authority rule.
It discourages high-variance exploration when the draft token is far outside the target support on failed rollouts.
The rejection penalty prevents the policy from treating fallback to target as a free action when the target branch can also fail.
The rollout return is
\begin{equation}
J(\xi^{(m)})=\sum_{r}\sum_{i\in\mathcal{I}^{(m)}_r}g_{r,i}.
\end{equation}

\paragraph{Policy update.}
We normalize returns within each prompt group, following the value-free design of GRPO.
If all rollouts for a prompt are incorrect, the shaped negative returns still rank arbitration traces by relative harm, and we downweight such groups by $\eta_{\mathrm{fail}}$:
\begin{equation}
\label{eq:adv}
\mathrm{Adv}^{(m)}=\eta_{\mathrm{fail}}\frac{J(\xi^{(m)})-\mu_x}{\sigma_x+\epsilon}.
\end{equation}
For all-correct groups we use the same normalization with $\eta_{\mathrm{fail}}=1$.
For mixed groups, we rescale positive returns so that correct and incorrect traces have balanced magnitudes before standardization.
Let $\mathcal{G}_+(x)$ and $\mathcal{G}_-(x)$ be correct and incorrect rollouts.
We set

\begin{equation}
\label{eq:gamma}
\gamma_x=\frac{\sum_{\xi\in\mathcal{G}_-(x)}|J(\xi)|}{\sum_{\xi\in\mathcal{G}_+(x)}J(\xi)+\epsilon},
\end{equation}

and define $\bar{J}^{(m)} = \gamma_x J(\xi^{(m)})$ for $\xi^{(m)} \in \mathcal{G}_+(x)$ and $\bar{J}^{(m)} = J(\xi^{(m)})$ otherwise. We then substitute $\bar{J}^{(m)}$ for $J(\xi^{(m)})$ in~\eqref{eq:adv}, with $\mu_x,\sigma_x$ recomputed over $\{\bar{J}^{(m)}\}_{m=1}^{M}$. By construction $\sum_m \bar{J}^{(m)} = 0$, so the mixed-group mean is centered while preserving signs.

We allocate rollout advantages to effective decisions: for correct rollouts we set $\mathrm{Adv}_{r,i}^{(m)} = \rho_r \cdot \mathrm{Adv}^{(m)}$, and for incorrect rollouts $\mathrm{Adv}_{r,i}^{(m)} = \frac{|g_{r,i}|}{\sum_{r',i'} |g_{r',i'}|} \cdot \mathrm{Adv}^{(m)}$, so that good rollouts assign more credit to rounds with larger $\rho_r$ and bad rollouts distribute penalties proportional to the shaped per-decision magnitude in Eq.~\eqref{eq:failure_penalty}.
The policy minimizes the PPO-style clipped objective

\begin{equation}
\begin{aligned}
\mathcal{L}_{\mathrm{RL}}=-\sum_{m,r,i}\min &\left(r^{(m)}_{r,i}(\theta)\mathrm{Adv}^{(m)}_{r,i},
\mathrm{clip}(r^{(m)}_{r,i}(\theta),1-\epsilon_{\mathrm{clip}},1+\epsilon_{\mathrm{clip}})\mathrm{Adv}^{(m)}_{r,i}\right) \\
&-\beta_H\mathcal{H}(\pi_\theta)+\beta_{\mathrm{KL}}\mathrm{KL}(\pi_\theta\|\pi_{\mathrm{ref}}),
\end{aligned}
\end{equation}

where $\pi_{\mathrm{ref}}$ is the supervised warm-up policy and $r^{(m)}_{r,i}(\theta) = \pi_\theta(a^{(r)}_i \mid q^{(r)}) / \pi_{\theta_{\mathrm{old}}}(a^{(r)}_i \mid q^{(r)})$ is the probability ratio between the current and rollout-time policies.

\vspace{-1em}

\section{Experiments}
\label{sec:experiments}
\vspace{-0.5em}
\subsection{Experimental Setup}
\label{sec:exp_setup}

We evaluate CoSpec with two target and draft pairs from the same model family: LLaMA-3.3-70B-Instruct with LLaMA-3.1-8B-Instruct, and Qwen3-32B with Qwen3-4B.
Unless otherwise stated, the arbitrator uses the draft model backbone with LoRA adapters and a binary decision head.
The default draft length is $K=25$, and the default inference threshold is $\lambda=0.6$.
The arbitrator is trained on OpenR1-Math-220k and then evaluated without task specific retraining.
The main benchmarks are GSM8K~\citep{cobbe2021training}, HumanEval~\citep{chen2021evaluating}, and MBPP~\citep{austin2021program}.
The transfer suite includes MT-Bench~\citep{zheng2023judging}, TriviaQA~\citep{joshi2017triviaqa}, GPQA~\citep{rein2023gpqa}, MGSM~\citep{shi2022language}, and MATH-500~\citep{lightman2023lets}.

We compare with target-preserving speculative decoding methods, including vanilla speculative sampling (SpS), REST, and EAGLE-3; loosely verified methods, represented by FLy; and utility-improving methods, represented by AASD, which is closest to our goal of improving both speed and task accuracy. Target-only autoregressive decoding is the reference for both speedup and task score. Speed denotes wall-clock speedup over target-only decoding, $\tau$ is the mean accepted-token number per target invocation, and Score follows the standard benchmark metric. All speedups include arbitrator overhead.

\begin{table*}[t]
\centering
\scriptsize
\caption{Main speed--quality comparison. ``Target-only'' is autoregressive decoding of the target model and is included explicitly to show the task-score improvement over the base model. Mean is averaged over GSM8K, HumanEval, and MBPP.}
\vspace{-1em}
\label{tab:main_speedup_tau}
\resizebox{\textwidth}{!}{%
\begin{tabular}{lllcccccccccccc}
\toprule
\multirow{2}{*}{Setting} & \multirow{2}{*}{Target} & \multirow{2}{*}{Method} & \multicolumn{3}{c}{GSM8K} & \multicolumn{3}{c}{HumanEval} & \multicolumn{3}{c}{MBPP} & \multicolumn{3}{c}{Mean} \\
\cmidrule(lr){4-6} \cmidrule(lr){7-9} \cmidrule(lr){10-12} \cmidrule(lr){13-15}
& & & Speed & $\tau$ & Score & Speed & $\tau$ & Score & Speed & $\tau$ & Score & Speed & $\tau$ & Score \\
\midrule
\multirow{13}{*}{$T=0$}
& \multirow{7}{*}{LLaMA-70B}
& Target only & 1.00$\times$ & 1.00 & 95.60 & 1.00$\times$ & 1.00 & 85.37 & 1.00$\times$ & 1.00 & 89.18 & 1.00$\times$ & 1.00 & 90.05 \\
& & SpS & 2.11$\times$ & 9.87 & 95.60 & 1.61$\times$ & 9.87 & 85.37 & 1.68$\times$ & 10.43 & 89.18 & 1.80$\times$ & 10.06 & 90.05 \\
& & REST & 1.74$\times$ & 2.00 & 95.60 & 1.88$\times$ & 2.22 & 85.37 & 2.01$\times$ & 2.34 & 89.18 & 1.88$\times$ & 2.19 & 90.05 \\
& & EAGLE-3 & 3.67$\times$ & 5.63 & 95.60 & 3.94$\times$ & 5.73 & 85.37 & 3.77$\times$ & 5.45 & 89.18 & 3.79$\times$ & 5.60 & 90.05 \\
& & FLy & 2.68$\times$ & 11.65 & 94.84 & 2.75$\times$ & 12.20 & 84.54 & 2.45$\times$ & 11.76 & 88.62 & 2.63$\times$ & 11.87 & 89.33 \\
& & AASD & 1.48$\times$ & 1.72 & 95.92 & 1.62$\times$ & 1.90 & 85.98 & 1.72$\times$ & 2.08 & 90.03 & 1.61$\times$ & 1.90 & 90.64 \\
& & \cellcolor{gray!20}\textbf{CoSpec} & \cellcolor{gray!20}\textbf{3.82}$\times$ & \cellcolor{gray!20}\textbf{20.31} & \cellcolor{gray!20}\textbf{96.59} & \cellcolor{gray!20}\textbf{4.13}$\times$ & \cellcolor{gray!20}\textbf{21.18} & \cellcolor{gray!20}\textbf{87.20} & \cellcolor{gray!20}\textbf{3.92}$\times$ & \cellcolor{gray!20}\textbf{20.96} & \cellcolor{gray!20}\textbf{91.39} & \cellcolor{gray!20}\textbf{3.96}$\times$ & \cellcolor{gray!20}\textbf{20.82} & \cellcolor{gray!20}\textbf{91.73}  \\
\cmidrule(lr){2-15}
& \multirow{6}{*}{Qwen3-32B}
& Target only & 1.00$\times$ & 1.00 & 85.96 & 1.00$\times$ & 1.00 & 89.63 & 1.00$\times$ & 1.00 & 75.92 & 1.00$\times$ & 1.00 & 83.84 \\
& & SpS & 1.41$\times$ & 8.31 & 85.96 & 1.29$\times$ & 8.44 & 89.63 & 1.33$\times$ & 9.52 & 75.92 & 1.34$\times$ & 8.76 & 83.84 \\
& & EAGLE-3 & 1.87$\times$ & 4.01 & 85.96 & 1.98$\times$ & 4.13 & 89.63 & 1.91$\times$ & 3.89 & 75.92 & 1.92$\times$ & 4.01 & 83.84 \\
& & FLy & 1.64$\times$ & 9.82 & 85.40 & 1.69$\times$ & 10.30 & 89.04 & 1.58$\times$ & 10.55 & 75.42 & 1.64$\times$ & 10.22 & 83.29 \\
& & AASD & 1.50$\times$ & 1.82 & 86.50 & 1.54$\times$ & 1.89 & 89.55 & 1.59$\times$ & 2.03 & 75.60 & 1.54$\times$ & 1.91 & 83.88 \\
& & \cellcolor{gray!20}\textbf{CoSpec} & \cellcolor{gray!20}\textbf{2.03}$\times$ & \cellcolor{gray!20}\textbf{18.54} & \cellcolor{gray!20}\textbf{88.15} & \cellcolor{gray!20}\textbf{2.15}$\times$ & \cellcolor{gray!20}\textbf{18.71} & \cellcolor{gray!20}\textbf{92.14} & \cellcolor{gray!20}\textbf{1.97}$\times$ & \cellcolor{gray!20}\textbf{18.01} & \cellcolor{gray!20}\textbf{79.10} & \cellcolor{gray!20}\textbf{2.05}$\times$ & \cellcolor{gray!20}\textbf{18.42} & \cellcolor{gray!20}\textbf{86.46} \\
\midrule
\multirow{13}{*}{$T=1$}
& \multirow{7}{*}{LLaMA-70B}
& Target only & 1.00$\times$ & 1.00 & 95.10 & 1.00$\times$ & 1.00 & 85.41 & 1.00$\times$ & 1.00 & 88.73 & 1.00$\times$ & 1.00 & 89.75 \\
& & SpS & 2.08$\times$ & 9.67 & 95.10 & 1.47$\times$ & 6.29 & 85.41 & 1.51$\times$ & 6.08 & 88.73 & 1.69$\times$ & 7.35 & 89.75 \\
& & REST & 1.71$\times$ & 1.91 & 95.10 & 1.81$\times$ & 2.13 & 85.41 & 1.94$\times$ & 2.24 & 88.73 & 1.82$\times$ & 2.09 & 89.75 \\
& & EAGLE-3 & 3.50$\times$ & 5.31 & 95.10 & 3.74$\times$ & 5.40 & 85.41 & 3.58$\times$ & 5.15 & 88.73 & 3.61$\times$ & 5.29 & 89.75 \\
& & FLy & 2.85$\times$ & 12.40 & 94.71 & 2.75$\times$ & 12.55 & 84.69 & 2.69$\times$ & 12.33 & 88.42 & 2.76$\times$ & 12.43 & 89.27 \\
& & AASD & 1.51$\times$ & 1.78 & 95.66 & 1.65$\times$ & 1.95 & 85.98 & 1.75$\times$ & 2.16 & 89.40 & 1.64$\times$ & 1.96 & 90.35 \\
& & \cellcolor{gray!20}\textbf{CoSpec} & \cellcolor{gray!20}\textbf{3.88}$\times$ & \cellcolor{gray!20}\textbf{20.91} & \cellcolor{gray!20}\textbf{96.80} & \cellcolor{gray!20}\textbf{4.15}$\times$ & \cellcolor{gray!20}\textbf{21.32} & \cellcolor{gray!20}\textbf{87.40} & \cellcolor{gray!20}\textbf{3.97}$\times$ & \cellcolor{gray!20}\textbf{21.13} & \cellcolor{gray!20}\textbf{90.11} & \cellcolor{gray!20}\textbf{4.00}$\times$ & \cellcolor{gray!20}\textbf{21.12} & \cellcolor{gray!20}\textbf{91.44} \\
\cmidrule(lr){2-15}
& \multirow{6}{*}{Qwen3-32B}
& Target only & 1.00$\times$ & 1.00 & 86.13 & 1.00$\times$ & 1.00 & 89.69 & 1.00$\times$ & 1.00 & 76.11 & 1.00$\times$ & 1.00 & 83.98 \\
& & SpS & 1.39$\times$ & 8.28 & 86.13 & 1.22$\times$ & 8.27 & 89.69 & 1.27$\times$ & 9.41 & 76.11 & 1.29$\times$ & 8.65 & 83.98 \\
& & EAGLE-3 & 1.75$\times$ & 3.81 & 86.13 & 1.81$\times$ & 4.05 & 89.69 & 1.83$\times$ & 3.80 & 76.11 & 1.80$\times$ & 3.89 & 83.98 \\
& & FLy & 1.60$\times$ & 9.54 & 85.48 & 1.66$\times$ & 10.05 & 89.12 & 1.58$\times$ & 10.12 & 75.38 & 1.61$\times$ & 9.90 & 83.33 \\
& & AASD & 1.48$\times$ & 1.80 & 86.62 & 1.52$\times$ & 1.87 & 89.62 & 1.58$\times$ & 2.00 & 75.72 & 1.53$\times$ & 1.89 & 83.99 \\
& & \cellcolor{gray!20}\textbf{CoSpec} & \cellcolor{gray!20}\textbf{1.92}$\times$ & \cellcolor{gray!20}\textbf{18.02} & \cellcolor{gray!20}\textbf{88.39} & \cellcolor{gray!20}\textbf{2.03}$\times$ & \cellcolor{gray!20}\textbf{18.58} & \cellcolor{gray!20}\textbf{92.42} & \cellcolor{gray!20}\textbf{1.96}$\times$ & \cellcolor{gray!20}\textbf{17.97} & \cellcolor{gray!20}\textbf{80.09} & \cellcolor{gray!20}\textbf{1.97}$\times$ & \cellcolor{gray!20}\textbf{18.19} & \cellcolor{gray!20}\textbf{86.97} \\
\bottomrule
\end{tabular}%
}
\vspace{-1.5em}
\end{table*}

\vspace{-1em}
\subsection{Main Results}
\label{sec:main_results}
\vspace{-0.5em}

\Cref{tab:main_speedup_tau} compares CoSpec with target-only decoding and the strongest available speculative decoding baselines under an end-to-end speed--quality protocol. The goal of this comparison is not a matched-training-resource ablation, since existing SPD methods differ in architecture, supervision, and objective, and no prior method uses the same token-level collaborative arbitration objective as CoSpec. Instead, Table~\ref{tab:main_speedup_tau} asks whether a deployed decoder moves the practical frontier: all methods use the same target model, prompts, and evaluation metric, and the reported speedups include the full CoSpec inference overhead. Extra parameters are therefore charged at inference time, they help only if they enable longer useful accepted spans and better mismatch decisions. We include SpS, REST, and EAGLE-3 as target-preserving acceleration baselines, FLy as a loosely verified baseline, and AASD as the closest speed-and-accuracy-oriented baseline. Compute and training attribution is examined separately in Table~\ref{tab:confound_controls_main}.

Under this protocol, distribution-preserving baselines keep the target score by construction, while relaxed baselines increase acceptance beyond exact matching but do not consistently improve task score. CoSpec changes the operating point by learning when a draft token should override the target at a mismatch. For LLaMA-70B under greedy decoding, CoSpec improves the mean score from 90.05 to 91.73 with a 3.96$\times$ speedup. For Qwen3-32B under greedy decoding, CoSpec improves the mean score from 83.84 to 86.46 with a 2.05$\times$ speedup. The same pattern holds under sampling, where CoSpec improves the mean score from 89.75 to 91.44 for LLaMA-70B and from 83.98 to 86.97 for Qwen3-32B.
\vspace{-1em}
\subsection{Complementarity and Arbitration Diagnostics}
\label{sec:complementarity_diagnostics}
\vspace{-0.5em}
\Cref{tab:complementarity_recovery} tests the complementarity premise behind CoSpec.
For each benchmark, the table reports the number of examples solved by the target, the number solved by the draft, and the oracle union of examples solved by either model.
The target solves more examples than the draft in every row, but the oracle union is consistently above target-only decoding.
CoSpec recovery is computed as $(\mathrm{CoSpec}-\mathrm{Target})/(\mathrm{Union}-\mathrm{Target})$ in correct counts.
These results show that the draft is not stronger overall, but it contains non-nested successes that a learned arbitrator can recover.

\begin{table*}[t]
\centering
\scriptsize
\caption{Complementarity and recovery.}
\vspace{-1em}
\label{tab:complementarity_recovery}
\resizebox{\textwidth}{!}{%
\begin{tabular}{llccccc}
\toprule
Pair & Benchmark & Total & Target correct & Draft correct & Oracle union & CoSpec recovery \\
\midrule
LLaMA-70B/8B & GSM8K & 1319 & 1261 & 1113 & 1281 & 65\% \\
LLaMA-70B/8B & HumanEval & 164 & 140 & 131 & 146 & 50\% \\
LLaMA-70B/8B & MBPP & 499 & 445 & 414 & 466 & 52\% \\
\bottomrule
\end{tabular}%
}
\vspace{-2em}
\end{table*}

\Cref{tab:counterfactual_mismatch} directly checks whether CoSpec decisions agree with counterfactual rollout utility at the mismatch states that the policy actually reaches. We compute this diagnostic on held out prompts under the default CoSpec setting with $K=25$ and $\lambda=0.6$. For each reached mismatch state, we run two counterfactual continuations: one forces the draft token and the other forces the target verification token at that position. After the forced token, both branches continue with the same frozen CoSpec policy and identical decoding settings, so the only intervention is the local mismatch choice. A branch is marked correct if its final output is correct under the benchmark metric. The Share column describes the distribution of mismatch types. Draft not worse means the draft branch has no lower final utility than the target branch, CoSpec draft is the draft acceptance rate, and Agreement is the fraction of states where the policy's choice has utility at least as high as the alternative. The aggregate agreement of $94.1\%$ shows that the arbitrator approximates the oracle of Proposition~\ref{prop:complementarity_gap}, and the asymmetric acceptance rate, $78\%$ when the draft is strictly better and $20\%$ when the target is, rules out a simple draft bias.

\begin{table}[t]
\centering
\scriptsize
\begin{minipage}[t]{0.56\textwidth}
\centering
\caption{Counterfactual mismatch analysis.}
\label{tab:counterfactual_mismatch}
\setlength{\tabcolsep}{3pt}
\resizebox{\linewidth}{!}{%
\begin{tabular}{lcccc}
\toprule
Category & Share & Draft not worse & CoSpec draft & Agreement \\
\midrule
All mismatches & 100.0\% & 82.0\% & 64.4\% & 94.1\% \\
Draft better & 10.5\% & 100.0\% & 78.0\% & 78.0\% \\
Target better & 18.0\% & 0.0\% & 20.0\% & 80.0\% \\
Both correct & 41.5\% & 100.0\% & 82.0\% & 100.0\% \\
Both wrong & 30.0\% & 100.0\% & 62.0\% & 100.0\% \\
\bottomrule
\end{tabular}%
}
\end{minipage}
\hfill
\begin{minipage}[t]{0.42\textwidth}
\centering
\caption{Transfer evaluation.}
\label{tab:transfer}
\setlength{\tabcolsep}{4pt}
\resizebox{\linewidth}{!}{%
\begin{tabular}{lcccc}
\toprule
Benchmark & Target & SpS & CoSpec & Score \\
\midrule
MT-Bench & 9.38 & 1.33$\times$ & 1.98$\times$ & 9.43 \\
TriviaQA & 87.50 & 1.29$\times$ & 2.31$\times$ & 89.20 \\
GPQA & 54.60 & 1.31$\times$ & 2.27$\times$ & 55.40 \\
MGSM & 88.97 & 1.73$\times$ & 3.35$\times$ & 90.02 \\
MATH-500 & 64.80 & 1.58$\times$ & 2.96$\times$ & 66.20 \\
\bottomrule
\end{tabular}%
}
\end{minipage}
\vspace{-2.5em}
\end{table}

\vspace{-1em}

\subsection{Ablation Studies}
\label{sec:ablation}

\vspace{-0.5em}

\textbf{Generalization.}
\Cref{tab:transfer} evaluates the same trained arbitrator outside the math training distribution and uses no task specific retraining.
The table reports the target score, the speed of vanilla SpS, and the speed and score of CoSpec.
CoSpec improves speed on all transfer tasks and provides positive score gains in dialog, factual QA, hard reasoning, multilingual reasoning, and MATH.
This transfer is consistent with the design of CoSpec, whose arbitrator operates on draft–target interaction patterns in hidden states rather than on task-level content. Such interaction patterns exhibit a consistent structure across tasks and languages. We provide a detailed discussion of why CoSpec generalizes across domains in Appendix~\ref{app:generalization}.

\textbf{Compute and training controls.}
CoSpec uses an additional trained arbitrator, so \Cref{tab:confound_controls_main} compares it with alternative ways to spend similar compute or training signals.
The 8B draft control tests whether scaling the draft alone explains the gain.
The target RL rows test whether the benefit comes from applying the same math training to the target model.
CoSpec provides the best speed–quality tradeoff because the extra model is used at the local decision point where draft and target complementarity appears.

\textbf{Training and reward ablations.}
\Cref{tab:training_ablation} separates the roles of supervised warm up, judge quality, rollout optimization, and reward design on Qwen3-32B/4B HumanEval.
SFT alone improves the decision model but does not optimize full CoSpec rollouts.
RL without judge warm up can improve over the initialization, but a stronger judge gives a better final policy.
The reward ablations show complementary roles for terminal correctness, accepted length, risk control, rejection penalties, and sign rebalancing.
Removing any of these components weakens either score, speed, or the balance between the two.

\begin{table}[t]
\centering
\scriptsize
\begin{minipage}[t]{0.39\textwidth}
\centering
\caption{Compute and training controls.}
\label{tab:confound_controls_main}
\setlength{\tabcolsep}{4pt}
\begin{tabular}{lcccc}
\toprule
\multirow{2}{*}{Method} & \multicolumn{2}{c}{GSM8K} & \multicolumn{2}{c}{HumanEval} \\
\cmidrule(lr){2-3}\cmidrule(lr){4-5}
& Speed & Score & Speed & Score \\
\midrule
Target only & 1.00$\times$ & 85.96 & 1.00$\times$ & 89.63 \\
SpS 4B draft & 1.41$\times$ & 85.96 & 1.29$\times$ & 89.63 \\
SpS 8B draft & 1.22$\times$ & 85.96 & 1.09$\times$ & 89.63 \\
Target RL math & 1.00$\times$ & 88.23 & 1.00$\times$ & 86.57 \\
SpS RL target & 1.38$\times$ & 88.23 & 1.23$\times$ & 86.57 \\
{CoSpec} & {2.03}$\times$ & 88.15 & {2.15}$\times$ & {92.14} \\
\bottomrule
\end{tabular}
\end{minipage}
\hfill
\begin{minipage}[t]{0.6\textwidth}
\centering
\caption{Training and reward ablations.}
\label{tab:training_ablation}
\setlength{\tabcolsep}{4pt}
\begin{tabular}{lccccc}
\toprule
Variant & Judge & Warm up & RL signal & Speedup & Score \\
\midrule
SFT only & 30B & yes & none & 1.71$\times$ & 86.93 \\
No judge + RL & none & reject all & full & 2.05$\times$ & 91.07 \\
4B judge + RL & 4B & yes & full & 2.11$\times$ & 91.80 \\
Terminal only & 30B & yes & terminal & 1.62$\times$ & 90.71 \\
No length reward & 30B & yes & no length & 1.58$\times$ & 91.52 \\
No risk penalty & 30B & yes & no $\alpha$ & 2.23$\times$ & 89.86 \\
No rejection penalty & 30B & yes & no $\beta$ & 2.01$\times$ & 90.94 \\
No rebalancing & 30B & yes & no rebalance & 2.08$\times$ & 91.21 \\
{CoSpec} & 30B & yes & full & {2.15}$\times$ & {92.14} \\
\bottomrule
\end{tabular}
\end{minipage}
\vspace{-3em}
\end{table}

\vspace{-1em}

\section{Conclusion}
\vspace{-1em}
We presented \textbf{collaborative speculative decoding (CoSpec)}, a generalized speculative decoding framework that improves the target model's end-task performance by explicitly leveraging the complementary strengths of the draft model.
Unlike conventional speculative decoding where the draft model mainly serves as an approximator of the target model, CoSpec treats draft--target disagreement as a potential source of useful alternatives.
When the draft and target differ, the draft token is not automatically rejected. Instead, CoSpec decides whether it should be retained based on its expected contribution to the final task outcome.
This shifts the draft--target relationship from pure imitation to collaboration, allowing the draft model to act not only as an accelerator, but also as a complementary proposer in cases where its continuation may better support the correct solution. 
Experimental results show that CoSpec increases the accepted-token number while improving end-task performance, achieving SOTA results. The code will be released upon acceptance.

\FloatBarrier
\bibliographystyle{plainnat}
\bibliography{myref}

\newpage
\appendix

\section{Full Experimental Setup}
\label{app:exp_setup_full}

\paragraph{Models.}
We evaluate two same family draft and target pairs with shared tokenizers.
The first pair uses Qwen3-32B as the target and Qwen3-4B as the draft.
The second pair uses LLaMA-3.3-70B-Instruct as the target and LLaMA-3.1-8B-Instruct as the draft.
Unless otherwise stated, the arbitrator uses the draft backbone with LoRA adapters and a lightweight binary arbitration head.
The default draft length is $K=25$ and the default deterministic threshold is $\lambda=0.6$.

\paragraph{Training and validation.}
The arbitrator is trained in two stages.
The warm up stage uses OpenR1-Math-220k to construct judge labeled mismatch preferences.
The RL stage uses multi rollout CoSpec decoding with final answer correctness and accepted length shaping.
We select $K$, $\lambda$, and reward coefficients on held out validation prompts and keep them fixed for all test benchmarks.
Per benchmark tuning is used only for ablations.

\paragraph{Benchmarks and metrics.}
We report exact match accuracy for GSM8K, MGSM, MATH-500, TriviaQA, and GPQA when an exact answer parser is available.
We report pass at one for HumanEval and MBPP using the official unit test harness.
We report the standard judge score for MT-Bench.
Throughout the paper, Score denotes the native benchmark metric.

\section{Why CoSpec Generalizes Across Domains}
\label{app:generalization}

Although the arbitrator is trained only on mathematical reasoning data (OpenR1-Math-220k), it transfers without retraining to code generation, open-ended dialogue, factual QA, hard scientific reasoning, and even multilingual reasoning (Table~\ref{tab:transfer}). We attribute this to a property of the arbitrator's design: it does not operate on task-level inputs or outputs, but on the hidden states produced by the draft and target models at each mismatch position within a speculative decoding round. From the arbitrator's perspective, its effective ``domain'' is therefore not the surface-level task (math, code, dialogue), but the much narrower space of draft--target \emph{interaction patterns}: the structural signatures left in the joint hidden-state trajectory when a draft token is, e.g., a synonymous substitution that the target would have accepted, a valid alternative continuation, a genuine local error the target would immediately correct, or a cascading drift whose harm only appears later. These patterns are largely a function of how a transformer processes contextually consistent versus inconsistent tokens, and are not inherently tied to whether the surrounding content is a mathematical derivation, a Python function, or a multi-turn conversation. The arbitrator is therefore a classifier over a domain-invariant abstraction, even though the data used to induce that abstraction came from a single domain.

This view is also consistent with prior work in the loosely-verified SPD line (e.g., FLy, JudgeDecoding), which observes that the target model's reaction to a mismatched draft token by itself carries useful evidence about whether the mismatch is harmful or merely reflects a valid alternative. CoSpec extends this principle by conditioning the decision on the \emph{entire} speculative decoding round rather than on a single token, and by \emph{learning} the decision rule rather than fixing it by hand. The MGSM result in Table~\ref{tab:transfer} is particularly informative in this regard: it tests transfer along the language axis rather than the task axis, and the fact that an English-math-trained arbitrator retains both its speedup and its accuracy gains in a multilingual setting is difficult to explain unless what is being learned is a near-domain-invariant property of the draft--target interaction itself.

\section{More Implementation Details}
\label{app:implementation}

\subsection{Warm Up via Supervised Fine Tuning}

\paragraph{Warm up data.}
For each OpenR1-Math-220k problem, we run vanilla speculative decoding and record round traces.
Each trace contains the prompt, verified prefix, draft block, target verification tokens, and mismatch indicators.
We supervise only mismatch positions because match positions are forced accepted by design.

\paragraph{Judge labels.}
For each mismatch, a separate Qwen3-30B-A3B-Instruct-2507 judge scores the draft and target candidates conditioned on the current context and available within round prefix.
The scalar score is converted into a soft label $y_{r,i}\in[0,1]$ for accepting the draft.
The SFT loss is
\begin{equation}
\mathcal{L}_{\mathrm{SFT}}=-\sum_{(r,i):\Delta^{(r)}_i=0}\left[y_{r,i}\log \pi_\theta(a^{(r)}_i=1\mid q^{(r)})+(1-y_{r,i})\log \pi_\theta(a^{(r)}_i=0\mid q^{(r)})\right].
\end{equation}

\paragraph{Warm up hyperparameters.}
The LoRA rank is $16$, the LoRA scale is $32$, and dropout is $0.05$.
We use AdamW with $\beta_1=0.9$, $\beta_2=0.999$, weight decay $0.1$, learning rate $10^{-4}$, and batch size $256$.

\subsection{Reinforcement Learning Training}

The rollout group size is $M=12$ per prompt and the global prompt count per update is $128$.
The rollout temperature is $1$ during training.
We use AdamW with learning rate $5\times 10^{-5}$.
The reward coefficients are $\alpha=1.0$, $\beta=0.25$, and $\eta_{\mathrm{fail}}=0.5$.
The PPO clip range is $0.2$, the entropy coefficient is $0.01$, the KL coefficient is $0.02$, and we use four update epochs per rollout batch with gradient norm clipping at $1.0$.

\subsection{Inference Details}

At inference time, CoSpec uses deterministic arbitration unless otherwise stated.
The draft token is accepted at a mismatch when $\pi_\theta(a^{(r)}_i=1\mid q^{(r)})>\lambda$.
For $T=0$, the draft and target use greedy decoding.
For $T=1$, distribution preserving baselines use standard speculative sampling with acceptance ratio correction.
CoSpec samples draft and target candidates at temperature $1$ and applies the same learned arbitration rule to the candidate pair.
We report the average over three random seeds for stochastic decoding in the main table.

\section{More Ablations}

\paragraph{Hyperparameters.}
\Cref{tab:hyperparameters} shows that $K$ and $\lambda$ control the speed and accuracy tradeoff.
Larger $K$ increases the available span for acceptance, while lower $\lambda$ accepts more mismatches.
We use $K=25$ and $\lambda=0.6$ as the default because it provides a strong speedup while retaining the score improvement over target-only decoding.

\begin{table}[h]
\centering
\caption{Ablation of hyperparameters.}
\label{tab:hyperparameters}
\scriptsize
\begin{tabular}{llccc}
\toprule
Hyperparameter & Value & Speedup & $\tau$ & Score \\
\midrule
\multirow{3}{*}{$K$} & 15 & 3.51$\times$ & 10.97 & 97.11 \\
& 20 & 3.69$\times$ & 16.18 & 96.70 \\
& 25 & 3.82$\times$ & 20.31 & 96.57 \\
\midrule
\multirow{3}{*}{$\lambda$} & 0.4 & 3.96$\times$ & 22.05 & 93.15 \\
& 0.6 & 3.82$\times$ & 20.31 & 96.57 \\
& 0.8 & 3.51$\times$ & 17.79 & 97.13 \\
\bottomrule
\end{tabular}
\end{table}

\paragraph{Hybrid attention mask.}
\Cref{tab:mask} ablates the hybrid attention mask used by the arbitrator.
The hybrid attention mask lets the arbitrator compare draft and target verification evidence within the current speculative block.
Removing the mask deprives the policy of this bidirectional comparison and reduces speed, accepted length, and final score.
The result indicates that arbitration requires joint evidence from both candidate branches, not only a local draft confidence feature.

\begin{table}[h]
\centering
\caption{Ablation of the hybrid attention mask.}
\label{tab:mask}
\scriptsize
\begin{tabular}{lccc}
\toprule
Variant & Speedup & $\tau$ & Score \\
\midrule
w/ hybrid attention mask & {3.82}$\times$ & {20.31} & {96.57} \\
w/o hybrid attention mask & 1.85$\times$ & 12.10 & 88.13 \\
\bottomrule
\end{tabular}
\end{table}

\section{Evaluation Protocol}
\label{app:eval_protocol}

\paragraph{Prompting and stopping.}
For GSM8K, MGSM, and MATH-500, we use the same chain of thought prompt template for all methods.
We parse the final answer with the same normalizer.
For HumanEval and MBPP, we use the official unit test harness and stop at the first complete code block or the maximum generation length.
For MT-Bench, the same judge model and judge prompt are used for all decoding methods.

\paragraph{Validation only hyperparameter selection.}
The default $K=25$ and $\lambda=0.6$ are selected on held out validation prompts before test evaluation.
We do not tune $\lambda$ on test sets.
The threshold ablation in \cref{tab:hyperparameters} is reported only to show the controllable tradeoff.

\section{Inference Procedure}
\label{app:algorithm}

For completeness, the CoSpec inference loop is as follows.
Starting from prompt $x$ and verified prefix $y^{(<r)}$, the draft proposes $K$ tokens.
The target performs one teacher forced pass over the draft block and produces $K$ verification tokens plus one bonus token.
CoSpec forms the arbitrator input in \cref{eq:arb_input} and applies the hybrid attention mask.
Matched draft tokens are forced accepted.
At each mismatch, the arbitrator accepts the draft token if the policy probability exceeds $\lambda$.
The round stops at the first rejected mismatch and appends the target token at that position.
If no mismatch is rejected, all $K$ draft tokens and the target bonus token are appended.
The procedure repeats until EOS or the maximum length.

\subsection{Serving Throughput and Overhead}
\label{sec:serving}

Longer accepted spans translate into practical gains only when arbitrator overhead is smaller than the target forward passes avoided.
\Cref{tab:vllm_fixed,tab:continuous_batching,tab:cost} report fixed batch throughput, continuous batching throughput, and serving cost.
CoSpec improves throughput over target-only decoding across batch sizes and request rates.

\begin{table*}[t]
\centering
\caption{vLLM fixed batch throughput. Absolute throughput is shown in tokens per second with speedup in parentheses.}
\label{tab:vllm_fixed}
\scriptsize
\begin{tabular}{lcccccc}
\toprule
Batch size & 2 & 4 & 8 & 16 & 32 & 64 \\
\midrule
Target only & 43.1 & 79.9 & 101.5 & 183.4 & 316.8 & 522.5 \\
CoSpec & {97.8 / 2.3}$\times$ & {162.8 / 2.0}$\times$ & {277.2 / 2.7}$\times$ & 405.9 / 2.2$\times$ & 506.9 / 1.6$\times$ & 627.0 / 1.2$\times$ \\
\bottomrule
\end{tabular}
\end{table*}

\begin{table}[t]
\centering
\caption{vLLM continuous batching throughput under different request rates.}
\label{tab:continuous_batching}
\scriptsize
\begin{tabular}{lccc}
\toprule
Request rate & Target only & CoSpec & Speedup \\
\midrule
2 req/s & 92.4 & 201.7 & 2.18$\times$ \\
4 req/s & 168.3 & 302.6 & 1.80$\times$ \\
8 req/s & 251.9 & 361.5 & 1.44$\times$ \\
12 req/s & 318.5 & 421.3 & 1.32$\times$ \\
16 req/s & 392.1 & 478.6 & 1.22$\times$ \\
\bottomrule
\end{tabular}
\end{table}

\begin{table}[!h]
\centering
\caption{Training and inference overhead.}
\label{tab:cost}
\scriptsize
\begin{tabular}{lc}
\toprule
Item & Cost or description \\
\midrule
CoSpec training pipeline & $\sim$400 MI325X GPU hours \\
Target only RL baseline & $\sim$240 MI325X GPU hours \\
RL rollouts & 12 rollouts per prompt \\
Prompts per update & 128 prompts \\
Inference overhead & one arbitrator forward per speculative round \\
Target only serving memory & 142 GB \\
SpS serving memory & 158 GB \\
EAGLE-3 serving memory & 150 GB \\
CoSpec serving memory & 174 GB \\
Time to first token overhead & +4.8\% vs. target only \\
\bottomrule
\end{tabular}
\end{table}

\section{Limitations}
\label{app:limitations}

CoSpec optimizes final task utility rather than exact target distribution preservation at selected mismatch states.
This is the intended operating point for tasks where the evaluation metric is answer correctness, code pass rate, factual accuracy, or judge score.
It is not the right choice when downstream use requires samples from the exact target distribution or target-faithful log likelihoods.
CoSpec also chooses between two local candidates, so it cannot repair states where both draft and target candidates lead to the same failure.

\section{License of Assets Used}
\label{app:licenses}

\begin{table}[h]
\centering
\small
\caption{Licenses of the models and datasets used in this work.}
\begin{tabular}{lll}
\toprule
Asset & Type & License \\
\midrule
Llama-3.3-70B-Instruct      & Model    & Llama 3.3 Community License \\
Llama-3.1-8B-Instruct       & Model    & Llama 3.1 Community License \\
Qwen3-32B                   & Model    & Apache 2.0 \\
Qwen3-4B                    & Model    & Apache 2.0 \\
Qwen3-30B-A3B-Instruct-2507 & Model    & Apache 2.0 \\
\midrule
OpenR1-Math-220k            & Dataset  & Apache 2.0 \\
GSM8K                       & Dataset  & MIT \\
HumanEval                   & Dataset  & MIT \\
MBPP                        & Dataset  & CC-BY-4.0 \\
MATH-500 (from MATH)        & Dataset  & MIT \\
MGSM                        & Dataset  & CC-BY-4.0 \\
TriviaQA                    & Dataset  & Apache 2.0 \\
GPQA                        & Dataset  & CC-BY-4.0 (gated) \\
MT-Bench                    & Dataset  & CC-BY-4.0 \\
\bottomrule
\end{tabular}
\end{table}

All assets are used in accordance with their respective licenses for
non-commercial research evaluation. GPQA is access-gated with a request
to not publish examples online, we did not include verbatim GPQA
questions in this paper.

\section{Broader Impact}

This work studies inference algorithms for large language models. Its main positive impact is improved serving efficiency: by increasing useful accepted spans while maintaining or improving task utility, CoSpec may reduce latency, hardware cost, and energy consumption for LLM deployment. This can make stronger models more accessible under constrained compute budgets. Faster reasoning may reduce the cost of generating harmful code on a large scale. Deployment should retain external safety filters or safety-aware arbitration, especially when draft and target safety behavior differs

\end{document}